%% file: main.tex
\documentclass[letterpaper, 10 pt, conference]{ieeeconf}  

\IEEEoverridecommandlockouts                              

\usepackage[pdftex]{graphicx}
\usepackage{subcaption}
\usepackage{hyperref}
\hypersetup{
    colorlinks=true,
    linkcolor=black,    
    urlcolor=blue,
    }
\usepackage{booktabs}
\usepackage{amsmath} 
\usepackage{amssymb}  
\usepackage{mathrsfs}
\usepackage{cleveref}
\usepackage{color}
\usepackage{calc}
\usepackage{dsfont}
\usepackage{bm}
\usepackage{float}

\newtheorem{theorem}{Theorem}
\newtheorem{proposition}{Proposition}

\newtheorem{definition}{Definition}

\newtheorem{remark}{Remark}
\usepackage{xspace}

\newcommand{\param}{\ensuremath{\bm{\theta}}\xspace}
\newcommand{\paramSet}{\ensuremath{\Theta}\xspace}
\newcommand{\paramTrue}{\ensuremath{\bm{\theta}^*}\xspace}
\newcommand{\paramEst}{\ensuremath{\hat{\bm{\theta}}}\xspace}
\newcommand{\paramError}{\ensuremath{\tilde{\bm{\theta}}}\xspace}
\newcommand{\paramErrorMax}{\ensuremath{\tilde{\bm{\vartheta}}}\xspace}

\DeclareMathOperator*{\argmin}{arg\,min}

\usepackage{amsfonts} 
\title{\LARGE \bf
Robust Adaptive Time-Varying Control Barrier Function\\ with Application to Robotic Surface Treatment}

\author{Yitaek Kim and Christoffer Sloth 
\thanks{Authors are with the Maersk Mc-Kinney Moller Institute, University of Southern Denmark, Denmark {\tt\small \{yik,chsl\}@mmmi.sdu.dk}}
}

\usepackage{tikz}

\newcommand\submittedtext{%
  \footnotesize \textcopyright \text{ }2025 IEEE.  Personal use of this material is permitted.  Permission from IEEE must be obtained for all other uses, in any current or future media, including reprinting/republishing this material for advertising or promotional purposes, creating new collective works, for resale or redistribution to servers or lists, or reuse of any copyrighted component of this work in other works.}
  
\newcommand\submittednotice{%
\begin{tikzpicture}[remember picture,overlay]
\node[anchor=south,yshift=10pt] at (current page.south) {\fbox{\parbox{\dimexpr\textwidth-\fboxsep-\fboxrule\relax}{\submittedtext}}};
\end{tikzpicture}%
}

\begin{document}
\maketitle
\submittednotice
\thispagestyle{empty}
\pagestyle{empty}


\begin{abstract}
Set invariance techniques such as control barrier functions (CBFs) can be used to enforce time-varying constraints such as keeping a safe distance from dynamic objects. However, existing methods for enforcing time-varying constraints often overlook model uncertainties. To address this issue, this paper proposes a CBFs-based robust adaptive controller design endowing time-varying constraints while considering parametric uncertainty and additive disturbances. To this end, we first leverage Robust adaptive Control Barrier Functions (RaCBFs) to handle model uncertainty, along with the concept of Input-to-State Safety (ISSf) to ensure robustness towards input disturbances. Furthermore, to alleviate the inherent conservatism in robustness, we also incorporate a set membership identification scheme. We demonstrate the proposed method on robotic surface treatment that requires time-varying force bounds to ensure uniform quality, in numerical simulation and real robotic setup, showing that the quality is formally guaranteed within an acceptable range. 
\end{abstract}

\input{introduction}

\input{problem_formulation}
\input{preliminary_knowledges}
\input{proposed_method}
\input{application}
\input{conclusions}

\bibliographystyle{IEEEtran}
\bibliography{bibliography}

\end{document}

%% file: introduction.tex
\section{Introduction}\label{sec:introduction}
Set invariance can be achieved in many ways, and using Control Barrier Functions (CBFs) \cite{Ames2019CBFtheoryandapplications} is one of the promising approaches, ensuring forward invariance of a given set. CBFs can be easily implemented with quadratic programming and have been widely used in robotics applications \cite{dawson2022barrier}\cite{Mitsioni2020}\cite{ykgpracbf}. Although CBFs accomplish the forward invariance of a given \textit{static} constraint set, they are not applicable to applications where a constraint set is changed over time: a time-varying set. In this vein, using Time-Varying Control Barrier Functions (TVCBFs) is one of the potential ways to enforce time-varying constraints and successfully applied to human-assist wheelchair control \cite{TvCBF2019} and dynamic obstacle avoidance \cite{Huang2023}. 

However, since CBF-based controllers are highly dependent on the model accuracy, the forward invariance of a given set might not be satisfied if the model uncertainties such as parametric uncertainty and input disturbances are not taken into account. To resolve this issue, adaptive and robust control schemes have been recently proposed. Robust adaptive CBFs (RaCBFs) \cite{lopez2020robust} introduce the tightened safe set by imposing robustness against the uncertainty and can be integrated with Set Membership IDentification (SMID) \cite{SMID1992} to estimate the maximum bound of the uncertainty, so that less conservatism can be achieved. Another uncertainty, input disturbances are frequent in real-world implementations; these deteriorate the performance of CBFs-based controllers. The notion of Input-to-State Safety (ISSf) is defined to treat input disturbances, making a given safe set inflated based on the magnitude of the disturbances \cite{ShishirISSfCBFs2019}; thus they enable minimizing safety violations. Notably, in \cite{TEZUKA202244}, TVCBFs are combined with ISSf conditions to impose robustness against the bounded input disturbances. However, they do not consider inherent model uncertainty in the system model like RaCBFs, so the violation of constraints could occur.

\begin{figure}[t]
    \centering
\includegraphics[width=1\linewidth]{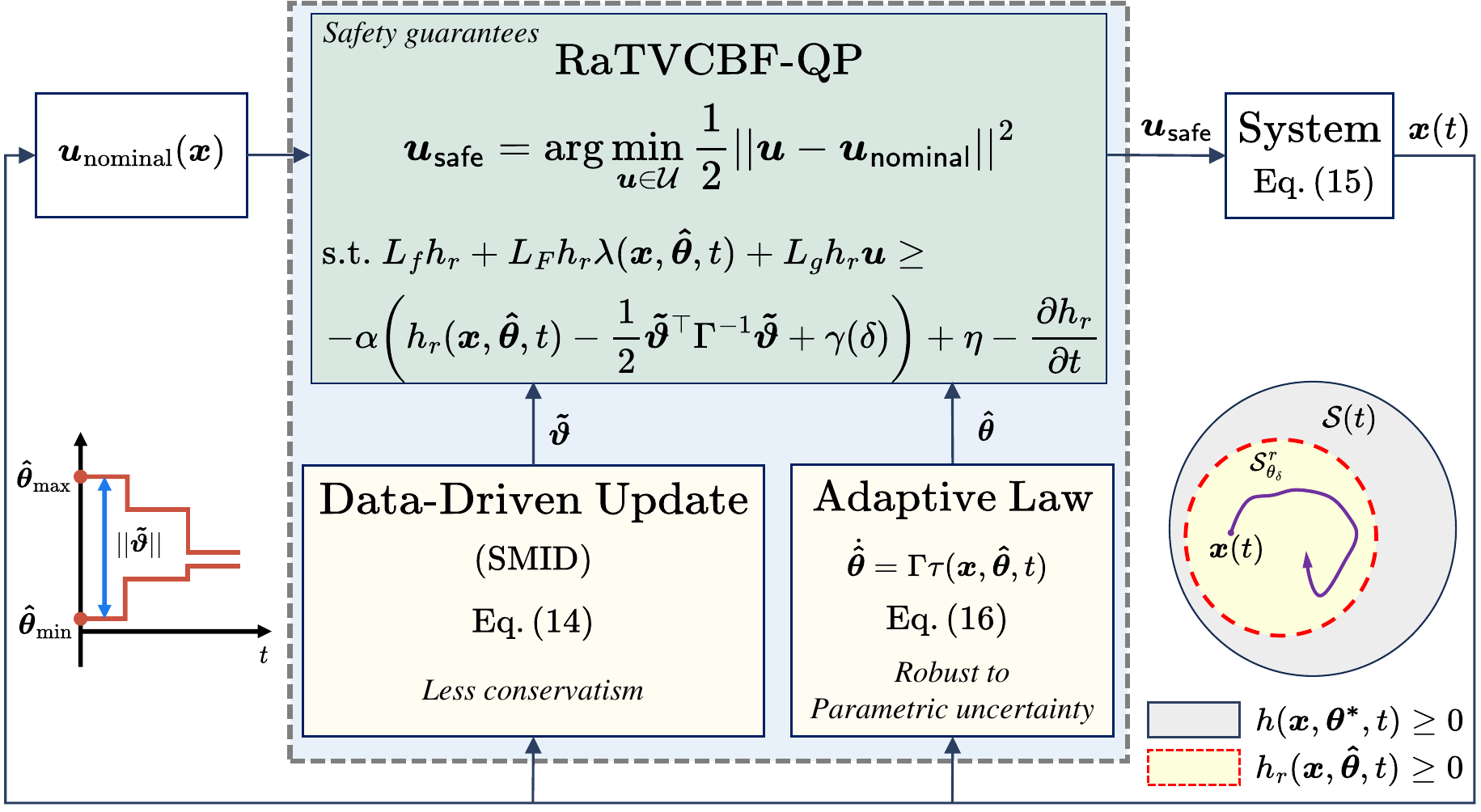}
    \caption{\small{Block diagram of the control system including a robust adaptive time-varying control barrier function with a given time-varying safe set, $\mathcal{S}(t)$. Given an uncertain system with parametric uncertainty and additive disturbances, the safety filter ensures that the system states stay in a robust subset, $\mathcal{S}_{\param_{\delta}}^{r} \subseteq \mathcal{S}(t)$. Data-driven update method, Set Membership IDentification (SMID) is used to reduce the maximum possible error $||\bm{\tilde{\vartheta}}||$.}}
    \vspace{-0.5cm}
    \label{fig:robust_safe_set}
\end{figure}

In this paper, we propose robust adaptive CBFs-based control to ensure set invariance of time-varying sets while considering parametric uncertainty and additive input disturbances. To this end, we use the ISSf condition on top of the RaCBF formulation to make a controller robust against not only input disturbances but also parametric model uncertainty. At the same time, we consider a time-varying control barrier function in the proposed control design to deal with a given set that varies over time. Lastly, we apply SMID to the proposed method to reduce conservatism.

We demonstrate the proposed method on robotic surface treatment application in this paper. Generally, the robotic surface treatment can be accomplished by maintaining a constant contact force between the polishing tool and the surface of a given workpiece \cite{Li}. Many force control schemes including passive \cite{Wang}\cite{Han} and active compliance controllers \cite{Buckmaster} have been used in robotic surface treatment. 

Even if the current robotic systems for surface treatment have focused on controlling constant contact forces to achieve an acceptable quality, they often neglect contact pressures over the surface, which leads to failure to keep uniform quality. Moreover, formal quality guarantees have not been fully investigated yet, which is a main challenge in robotic surface treatment applications \cite{YkCASE2022}\cite{XIAO2020188}. To remedy this challenge, the proposed method is suitable for the robotic surface treatment application to accomplish quality guarantees based on a set invariance scheme. This is achieved by the fact that the proposed method regulates the time-varying reference contact force trajectory, which in turn governs the contact pressure, and then we ultimately regulate the quality of robotic surface treatment.

The contributions in this paper are twofold as follows:
\begin{itemize}
    \item We propose the robust adaptive CBFs-based control design to deal with time-varying constraints, parametric uncertainty, and input disturbances as shown in Fig.~\ref{fig:robust_safe_set}. 
    \item We successfully demonstrate the proposed method on robotic surface treatment application to ensure quality guarantees in simulation and real robotic setup. 
\end{itemize}

The paper is constructed as follows. Section~\ref{sec:problem_formulation} presents a motivating example in the paper. Section~\ref{sec:preliminary_knowledges} provides the preliminaries about RaCBFs, ISSf, and time-varying control barrier functions. The proposed method is presented in Section~\ref{sec:method}. Applying the proposed method to the robotic surface treatment is described in Section~\ref{sec:application}, verified in the numerical simulation and experimentally validated on the real robotic setup with clear comparison results to an existing method. Lastly,  Section~\ref{sec:conclusions} concludes this paper.

%% file: problem_formulation.tex
\section{Motivating Example}\label{sec:problem_formulation}
In this section, we describe an application example that motivates this work. The proposed method is inspired by a robotic surface treatment task where the quality can be determined based on the material removal rate (MRR) \cite{Li}. The MRR is modeled using Preston's equation \cite{Preston}:
\begin{equation}
\texttt{MRR} = k_{p}\frac{F_c}{A}w, \label{Preston_equation}
\end{equation}
where $k_p>0$ is the Preston coefficient (material dependent), $F_c$ is the contact force between the tool and the surface with area $A$, and $w\geq0$ is the speed of the tool relative to the surface.

To guarantee the quality of a surface treatment task, the MRR should be within given bounds. This can be realized in different ways e.g. by changing the speed $w$ or by changing the force $F_c$; in this work, only the force is changed. This implies that the force should be controlled to comply with the given bounds during the entire task, despite a changing contact area; this leads to time-varying force bounds as shown in Fig.~\ref{fig:mrr_motivation}. 

In this example, the bounds on the force vary due to a changing contact area, while the bounds on MRR are constant. 
In addition, the system model is uncertain due to the Preston coefficient being partially unknown, and uncertainties in the contact model; thus, we consider both parametric uncertainties and uncertainty given by an unknown but bounded additive input disturbance. 
\begin{figure}[t]
    \centering
\includegraphics[width=0.8\linewidth]{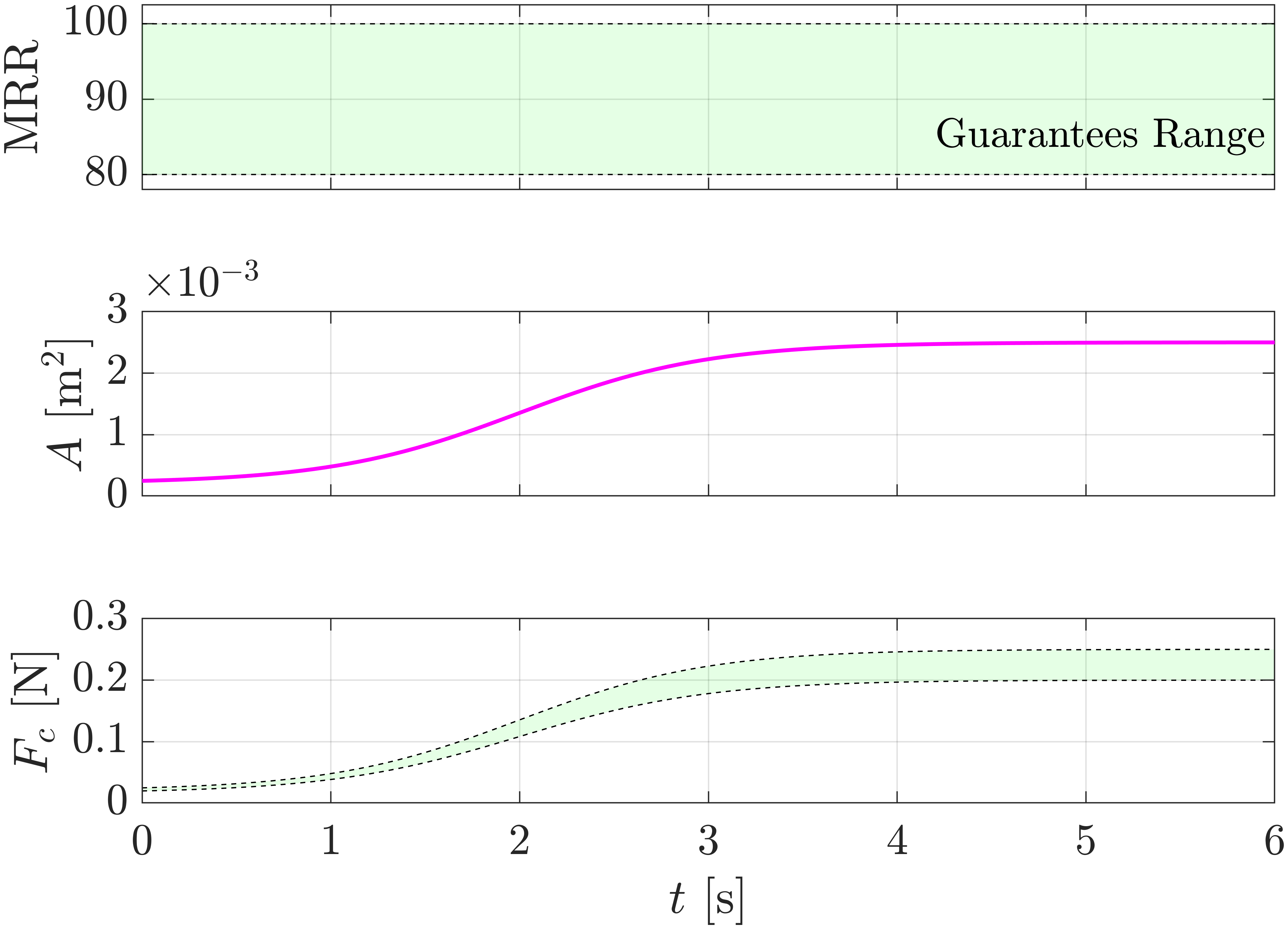}
    \caption{\small{Admissible time-varying force bounds to guarantee the desired MRR range.}}
    \vspace{-0.5cm}
    \label{fig:mrr_motivation}
\end{figure}

In summary, a set invariance method for ensuring quality guarantees via MRR with time-varying constraints despite model uncertainty (e.g. additive disturbance and parametric uncertainty) should be developed.

%% file: preliminary_knowledges.tex
\section{Preliminaries}\label{sec:preliminary_knowledges}
In this section, we revisit the concepts of safety formulated by Time-Varying Control Barrier Functions, robust adaptive CBFs to parametric uncertainty, Input-to-State Safety (ISSf) to deal with input disturbances in the system, and lastly Set Membership IDentification (SMID) to reduce conservatism of the proposed method.

\subsection{Time-Varying Control Barrier Functions}
Consider the control affine system:
\begin{equation}
        \dot{\bm{x}} = f(\bm{x}) + g(\bm{x})\bm{u},
        \label{eqn:sys}
\end{equation}
where $\bm{x} \in \mathcal{X} \subset \mathbb{R}^n$ is the state, $f,g$ are locally Lipschitz continuous functions, $\bm{u} \in \mathcal{U} \subset \mathbb{R}^m$ represents a control input. Consider a time-varying safety set $\mathcal{S}$ defined as $\mathcal{S} = \{\bm{x} \in \mathcal{X} \text{ }\vert\text{ } h(\bm{x},t) \geq 0\}$. A continuously differentiable function $h: \mathcal{X}\times \mathbb{R} \rightarrow \mathbb{R}$ is a Time-Varying Control Barrier Function (TVCBF) if there exist constants $K > 0$ and $C > 0$ satisfying the following for all $t\geq0$ \cite{TvCBF2019}:
\begin{align}
        \sup_{\bm{u}\in\mathcal{U}}\{L_fh(\bm{x},t) + &L_gh(\bm{x},t)\bm{u}\} + \frac{\partial h}{\partial t}(\bm{x},t) \nonumber \\ 
            &\geq -K(h(\bm{x},t)) -C \quad \forall  \bm{x}\in\mathcal{S}(t). \label{pre:cbf_condition}
\end{align}
If $h$ is a TVCBF for \eqref{eqn:sys}, then any locally Lipschitz continuous controller $\bm{u}$ satisfying the following for all $\forall \bm{x}\in\mathcal{S}(t)$:
\begin{equation}
    L_fh(\bm{x},t) + L_gh(\bm{x},t)\bm{u} + \frac{\partial h}{\partial t} \geq -K(h(\bm{x},t)) -C \label{pre:safety_condition}
\end{equation}
ensures that $x(t)\in\mathcal{S}(t)$ for all $t\geq0$.

\subsection{Input-to-State Safety Control Barrier Functions}
Unmodeled dynamics and disturbances may result in safety violations if we directly use standard CBFs \cite{Ames2019CBFtheoryandapplications}. We now introduce the input disturbance, $\bm{d} \in \mathbb{R}^m, ||\bm{d}||_{\infty}\leq \delta, \delta \geq 0$ in \eqref{eqn:sys}: 
\begin{equation}
        \dot{\bm{x}} = f(\bm{x}) + g(\bm{x})(\bm{u}+\bm{d}),
        \label{eqn:sys_input_disturbance}
\end{equation}
and define the larger safe set as:
\begin{equation}
    \mathcal{S}_{\delta} = \{\bm{x} \in \mathcal{X} \text{ }\vert\text{ } h_{\delta}(\bm{x}) \geq 0\},
\end{equation}
where $\mathcal{S} \subseteq \mathcal{S}_{\delta}$ and $h_{\delta}(\bm{x}) = h(\bm{x}) + \gamma(\delta)$. We say that the system \eqref{eqn:sys_input_disturbance} is \textit{Input-to-State Safe} (ISSf) if there exists $\gamma$ and $\delta$ such that for all $\bm{d}(t)$, $\mathcal{S}_{\delta}$ is forward invariant \cite{ShishirISSfCBFs2019}. An ISSf control barrier function is defined as follows:
\begin{theorem}[\cite{ShishirISSfCBFs2019}]
    A function $h$ is a \textit{Input-to-State Safe Control Barrier Function} (ISSf-CBF) if there exists extend class $\mathcal{K}_{\infty}$ functions $\alpha$ and $\iota$ such that $\forall  \bm{x}\in\mathcal{S}_{\delta}$: 
    \begin{equation}
        \sup_{\bm{u}\in\mathcal{U}}\{L_fh(\bm{x}) + L_gh(\bm{x})(\bm{u}+\bm{d})\} 
            \geq -\alpha(h(\bm{x})) - \iota(||\bm{d}||_{\infty}). \label{pre:issfcbf_condition_1}
    \end{equation}
\end{theorem}
If a function, $h$ is ISSf-CBF, and we choose control input $\bm{u}$ satisfying \eqref{pre:issfcbf_condition_1}, then the system \eqref{eqn:sys_input_disturbance} is ISSf with respect to $\mathcal{S}_{\delta}$.
\subsection{Robust adaptive Control Barrier Functions (RaCBFs)}
In the presence of parametric uncertainty in the system, we consider the following control affine system and adaptive parameter estimator dynamics simultaneously \cite{Taylor2020}: 
\begin{align}
    \begin{bmatrix}
    \dot{\bm{x}} \\
    \dot{\paramEst}
    \end{bmatrix} = \begin{bmatrix}
    f(\bm{x}) + F(\bm{x})\paramTrue+g(\bm{x})\bm{u} \\
    \Gamma\tau(\bm{x},\paramEst)
    \end{bmatrix} \label{pre:adaptive_ds} \\
    \text{with} \quad \tau(\bm{x}, \paramEst) = -\Bigg(\frac{\partial h_r}{\partial \bm{x}}(\bm{x}, \paramEst)F(\bm{x})\Bigg)^{\top}, \label{pre:tau}
\end{align}
where $\paramTrue \in \Theta \subset \mathbb{R}^k$ is the unknown parameter, and $\paramEst \in \paramSet$ is the estimated parameter, and $F:\mathcal{X}\rightarrow \mathbb{R}^{n \times k}$ is smoothly continuous on $\mathcal{X}$ such that $F(\textbf{0})=\textbf{0}$ , and $\tau$ is the parameter update rule, and $\Gamma$ is an adaptive gain. The parameterized safe set $\mathcal{S}_{\param}$ is defined as $\mathcal{S}_{\param} = \{\bm{x}\in\mathcal{X} \text{ } \vert \text{ } h_r(\bm{x},\param) \geq 0\}$ where $h_r: \mathcal{X} \times \paramSet \rightarrow \mathbb{R}$ is a continuously differentiable function. To achieve robustness against the parameter uncertainty, the tightened safety subset, $\mathcal{S}_{\param}^{r} \subseteq \mathcal{S}_{\param}$ is defined as:
\begin{align}
    \mathcal{S}_{\param}^{r} = \{\bm{x}\in\mathcal{X}:h_r(\bm{x},\param)\geq\frac{1}{2}\paramErrorMax^{\top}\Gamma^{-1}\paramErrorMax\}, \label{racbf_condition}
\end{align}
where $\paramErrorMax \in\mathbb{R}^k$ is a maximum possible parameter estimation error of each element of $\paramTrue - \paramEst$, and  $\Gamma \in \mathbb{R}^{n \times n}$ is an adaptive gain. Then, safety can be ensured based on the following theorem:
\begin{theorem}[\cite{lopez2020robust}]
A function $h_r(\bm{x},\param)$ is a RaCBF if there exists an extended class $\mathcal{K}_{\infty}$ function $\alpha$ 
 and locally Lipschitz continuous control input $\bm{u}$ such that for the system \eqref{pre:adaptive_ds} and any $\param \in \paramSet$, and for all $\bm{x}\in \mathcal{S}^r_{\param}$:
\begin{align}
    \begin{aligned}
    \sup_{\bm{u}\in\mathcal{U}}\Bigg[\frac{\partial h_r}{\partial \bm{x}}(\bm{x},\param)\bigg(&\mathcal{F}(\bm{x},\param)+g(\bm{x})\bm{u}\bigg)\Bigg] \\
    &\geq -\alpha\bigg(h_r(\bm{x},\param)-\frac{1}{2}\paramErrorMax^{\top}\Gamma^{-1}\paramErrorMax\bigg) \label{racbf_constraint}
    \end{aligned}
\end{align}
with
\begin{align}
    \qquad &\mathcal{F}(\bm{x},\param) = f(\bm{x}) + F(\bm{x})\lambda(\bm{x},\param),\label{pre:F}   \\
    &\lambda(\bm{x},\param) \triangleq {\param} - \Gamma\Bigg(\frac{\partial h_r}{\partial \param}(\bm{x},\param)\Bigg)^{\top},\label{pre:lambda} \\
    & \lambda_{\text{min}}(\Gamma) \geq \frac{\vert\vert \paramErrorMax \vert\vert^2}{2h_r(\bm{x},\param)}. \nonumber
\end{align}
If a such function exists then the system \eqref{pre:adaptive_ds} is safe w.r.t $\mathcal{S}^r_{\param}$.
\end{theorem}
\subsection{Set Membership IDentification (SMID)}
Set Membership IDentification (SMID) is leveraged to compute a  bound on the parameter estimation error based on input-output data \cite{SMID1992}. We consider a linear regression equation
\[
\mathcal{Y}_j=\mathcal{D}_j\bm{\theta}+d_j,
\]
where $d_j$ is a disturbance,  $\bm{\theta}$ is a vector of unknown constant parameters, and $\mathcal{Y}_j$, $\mathcal{D}_j$ are measured quantities.
Let $\mathcal{H} = \{\mathcal{Y}_j, \mathcal{D}_j\}^M_{j=1}$ be a data history with size  $M\in\mathbb{N}$. To bound the value of $\bm{\theta}^*$ in \eqref{pre:adaptive_ds}, we define $\mathcal{Y}_{j} = \dot{\bm{x}}(t_j) - f(\bm{x}(t_j)) - g(\bm{x}(t_j))\bm{u}(t_j)$ and $\mathcal{D}_j=F(x(t_j))$, where $t_j$ is the time at which the data is recorded. Consider the following sequence of bounded parameter sets with $\Xi_0 = [\underline{\bm{\theta}}_0,\overline{\bm{\theta}}_0]$ and $\forall j\in M$: 
\begin{align}
\Xi_{t_j} =\{\bm{\theta}\in\Xi_{t_{j-1}}| -\varepsilon\bm{1}_n \leq \mathcal{Y}_{j} - \mathcal{D}_{j}\bm{\theta} \leq \varepsilon\bm{1}_n\}, 
\end{align}
where $[\underline{\bm{\theta}}_0,\overline{\bm{\theta}}_0]$ is the initial lower and upper bound of the unknown parameter, and $\bm{1}_n$ is an $n-$dimensional vector of ones, and $\varepsilon \in \mathbb{R}$ is the precision variable that must be chosen such that $||d_j||_\infty\leq\epsilon$ for all $j$. The sequence of sets, $\Xi_{t_j}$ can be calculated based on Linear Programming (LP) such that $\Xi_{t_j} \subseteq \Xi_{t_{j-1}} \subseteq \cdots\subseteq\Xi_{0}$ and $\bm{\theta} \in \Xi_{t_j}$ \cite{smidbook}.

%% file: proposed_method.tex
\section{Proposed Method}\label{sec:method}
In this section, we propose a robust adaptive safe controller that is robust against input disturbance and parametric model uncertainty simultaneously. To this end, we consider the ISSf condition on the formulation of RaCBF. At the same time, we take into account a time-varying control barrier function in the controller design in an integrated formulation and lastly provide proof of safety.

\subsection{Parametric Adaptive System Model with Uncertainties}
Consider the following control affine system with parameter estimation dynamics and input disturbances:
\begin{align}
    \begin{bmatrix}
    \dot{\bm{x}} \\
    \dot{\paramEst}
    \end{bmatrix} = \begin{bmatrix}
    f(\bm{x}) + F(\bm{x})\paramTrue+g(\bm{x})(\bm{u}+\bm{d}(t)) \\
    \Gamma\tau(\bm{x},\paramEst,t)
    \end{bmatrix}  \label{pf:sys_real_input_d}\\
    \text{with} \quad \tau(\bm{x}, \paramEst,t) = -\Bigg(\frac{\partial h_r}{\partial \bm{x}}(\bm{x}, \paramEst,t)F(\bm{x})\Bigg)^{\top}, \label{pre:tau_raissf}
\end{align}
where $\bm{x} \in \mathcal{X} \subset \mathbb{R}^n$ is the state, $f,g$ are locally Lipschitz continuous functions, $\bm{u} \in \mathcal{U} \subset \mathbb{R}^m$ represents a control input, and $\bm{d}(t) \in \mathbb{R}^m$ is a bounded input disturbance, $||\bm{d}(t)||_{\infty}\leq \delta \in \mathbb{R}_{+}$, and $F:\mathcal{X}\rightarrow \mathbb{R}^{n \times k}$ is smooth on $\mathcal{X}$ and models parametric uncertainty for $f$ such that $F(\bm{0})=\bm{0}$, and $\paramTrue \in \paramSet \subset \mathbb{R}^k$ is the vector of unknown parameters. In the following, we include the time derivative term of $h_r$ to deal with time-varying constraints.

\subsection{Robust Adaptive Time-Varying Control Barrier Function}
We define the following time-varying robust parameterized larger safe set, $\mathcal{S}^r_{\theta} \subseteq \mathcal{S}^r_{\theta_{\delta}}$: 
\begin{equation}
    \mathcal{S}^r_{\theta_{\delta}} = \{\bm{x}\in\mathcal{X}:h_r(\bm{x},\param,t)-\frac{1}{2}\paramErrorMax^{\top}\Gamma^{-1}\paramErrorMax +\gamma(\delta)\geq 0\},\label{racbfissf_condition}
\end{equation}
where $\gamma(\cdot)$ is the extended class $\mathcal{K}_{\infty}$ function. And if we choose $\lambda_{\text{min}}(\Gamma) \geq \frac{\vert\vert \paramErrorMax \vert\vert^2}{2h_r(\bm{x},\param,t)}$ to satisfy $h = h_r -\frac{1}{2}\paramError^{\top}\Gamma^{-1}\paramError \geq h_r -\frac{1}{2}\paramErrorMax^{\top}\Gamma^{-1}\paramErrorMax \geq 0$  according to \cite{lopez2020robust},  then we obtain the following definition:

\begin{definition}
    A continuously differentiable function $h$ is a \textit{Robust adaptive Time-Varying Control Barrier Function} (RaTVCBF) for the system \eqref{pf:sys_real_input_d} with bounded disturbance, $||\bm{d}(t)||_{\infty}\leq \delta$ if there exist extended class $\mathcal{K}_{\infty}$ functions $\alpha$ and $\gamma$ such that for any $\param \in \paramSet$ and for all $\bm{x} \in \mathcal{S}^r_{\param_{\delta}}$:
\begin{align}
    &\sup_{\bm{u}\in\mathcal{U}}\Bigg[\frac{\partial h_r}{\partial \bm{x}}(\bm{x},\param,t)\bigg(\mathcal{F}(\bm{x},\param,t)+g(\bm{x})\bm{u}\bigg)\Bigg]\nonumber \\
    &\geq -\alpha\bigg(h_r(\bm{x},\param,t) - \frac{1}{2}\paramErrorMax^{\top}\Gamma^{-1}\paramErrorMax+\gamma(\delta)\bigg) +\eta  - \frac{\partial h_r}{\partial t}\label{raissf_constraint} \\
    &\text{with }\qquad \eta = \Big|\Big|L_gh_r(\bm{x},\param,t)\Big|\Big|\delta, \nonumber
\end{align}
where $\mathcal{F}(\bm{x},\param,t) = f(\bm{x}) + F(\bm{x})\lambda(\bm{x},\param,t)$ and $\lambda(\bm{x},\param,t) \triangleq {\param} - \Gamma\Big(\frac{\partial h_r}{\partial \param}(\bm{x},\param,t)\Big)^{\top}$.
\end{definition}
 We now present the proposition for input-to-state safety for time-varying constraints in the presence of input disturbances and parametric uncertainties.
\begin{proposition} 
Let $h$ be a RaTVCBF, and the functions $f,g$ are locally Lipschitz continuous. Then, any locally Lipschitz continuous controller $\bm{u}$ satisfying:
\begin{align}
       L_fh_r&+ L_Fh_r\lambda(\bm{x},\paramEst,t)+L_gh_r\bm{u} - \eta + \frac{\partial h_r}{\partial t} \nonumber \\ 
    &\geq -\alpha\bigg(h_r(\bm{x},\paramEst,t) - \frac{1}{2}\paramErrorMax^{\top}\Gamma^{-1}\paramErrorMax + \gamma(\delta)\bigg),
        \label{pro:safe_constr}
\end{align}
where $\eta = \Big|\Big|L_gh_r(\bm{x},\param,t)\Big|\Big|\delta$, $L_fh_r = \frac{\partial h_r}{\partial \bm{x}}(\bm{x},\paramEst,t)f(\bm{x})$, $L_Fh_r = \frac{\partial h_r}{\partial \bm{x}}(\bm{x},\paramEst,t)F(\bm{x})$, and $L_gh_r = \frac{\partial h_r}{\partial \bm{x}}(\bm{x},\paramEst,t)g(\bm{x})$, renders the system \eqref{pf:sys_real_input_d} ISSf with respect to $\mathcal{S}_{\param_{\delta}}^{r}$.
\label{proposition1}
\end{proposition}

\begin{proof} 
The proof follows \cite{lopez2020robust}. Let a candidate control barrier function $h$ be:
\begin{equation}
    h(\bm{x},\paramEst,t) = h_r(\bm{x},\paramEst,t) -  \frac{1}{2}\paramError^{\top}\Gamma^{-1}\paramError,
\end{equation}
where $\paramError =\paramTrue -  \paramEst$, and its derivative is
\begin{equation}
    \dot{h}(\bm{x},\paramEst,t) = \dot{h}_r(\bm{x},\paramEst,t) + \paramError^{\top}\Gamma^{-1}\dot{\paramEst} \label{pf:dot_h} \\
\end{equation}
with
\begin{align}
    \dot{h}_r = &\frac{\partial h_r}{\partial \bm{x}}(\bm{x},\paramEst,t)\bigg(f(\bm{x}) + F(\bm{x})\paramTrue+g(\bm{x})(\bm{u}+\bm{d})\bigg) \nonumber \\
    &+\frac{\partial h_r}{\partial \paramEst}(\bm{x},\paramEst,t)\dot{\paramEst}+\frac{\partial h_r}{\partial t}(\bm{x},\paramEst,t) \nonumber
\end{align}
For the sake of brevity, we replace $\frac{\partial h_r}{\partial \bm{x}}(\bm{x},\paramEst,t)$ and $\frac{\partial h_r}{\partial\paramEst}(\bm{x},\paramEst,t)$ with $\frac{\partial h_r}{\partial \bm{x}}$ and $\frac{\partial h_r}{\partial \paramEst}$, respectively.
We add and subtract the following term in \eqref{pf:dot_h}: 
\begin{equation}
    \frac{\partial h_r}{\partial \bm{x}}F(\bm{x})\Bigg(\paramEst -\Gamma\bigg(\frac{\partial h_r}{\partial\paramEst} \bigg)^{\top}\Bigg).
\end{equation}
Then, we obtain the following equation and inequalities with $\lambda(\bm{x},\paramEst,t) \triangleq {\paramEst} - \Gamma\big(\frac{\partial h_r}{\partial \paramEst}(\bm{x},\paramEst,t)\big)^{\top}$ from \cite{Taylor2020}:
\begin{align}
    \dot{h} &= \frac{\partial h_r}{\partial \bm{x}}\bigg(f(\bm{x}) + F(\bm{x})\lambda(\bm{x},\paramEst,t)+g(\bm{x})\bm{u}\bigg) \nonumber\\
    &\quad +\bigg( \frac{\partial h_r}{\partial\paramEst}\Gamma +{\paramError}^{\top}\bigg)\Bigg( \bigg(\frac{\partial h_r}{\partial \bm{x}}F(\bm{x})\bigg)^{\top} + \tau(\bm{x},\paramEst,t)\Bigg) \nonumber \\
    &\quad +\frac{\partial h_r}{\partial \bm{x}}g(\bm{x})\bm{d} + \frac{\partial h_r}{\partial t}(\bm{x},\paramEst,t)\nonumber \\
    &\geq \frac{\partial h_r}{\partial \bm{x}}\bigg(f(\bm{x}) + F(\bm{x})\lambda(\bm{x},\paramEst,t)+g(\bm{x})\bm{u}\bigg) \nonumber\\ 
    &\quad + \bigg( \frac{\partial h_r}{\partial \paramEst}\Gamma +{\paramError}^{\top}\bigg)\Bigg( \bigg(\frac{\partial h_r}{\partial \bm{x}}F(\bm{x})\bigg)^{\top} + \tau(\bm{x},\paramEst,t)\Bigg) \nonumber \\
    &\quad -\Big|\Big|L_gh_r\Big|\Big|\delta + \frac{\partial h_r}{\partial t}(\bm{x},\paramEst,t)\label{pf:last_inequ}. 
\end{align}

In the last inequality of \eqref{pf:last_inequ}, we choose $\tau(\bm{x},\paramEst,t) = -\Big(\frac{\partial h_r}{\partial \bm{x}}F(\bm{x})\Big)^{\top}$ similar to \eqref{pre:tau} provided by \cite{Taylor2020}, and then the second term will be zero. Afterwards, we introduce the maximum possible error, $\paramErrorMax$, and if the following inequality is satisfied: 
\begin{align}
&\frac{\partial h_r}{\partial \bm{x}}\bigg(f(\bm{x}) + F(\bm{x})\lambda(\bm{x},\paramEst,t)+g(\bm{x})\bm{u}\bigg) - \eta + \frac{\partial h_r}{\partial t} \nonumber\\
&\geq -\alpha\bigg(h_r(\bm{x},\paramEst,t) - \frac{1}{2}\paramErrorMax^{\top}\Gamma^{-1}\paramErrorMax + \gamma(\delta)\bigg), \label{prof:safe_cond}
\end{align}
where $\eta = \Big|\Big|L_gh_r \Big|\Big|\delta$, and we choose $\gamma(\delta) = -\alpha^{-1}(-\frac{\epsilon\delta^2}{4})$ from \cite{AlanTISSfCBFs2022}, and then finally the following sufficient condition of forward invariance of the set, $\mathcal{S}^r_{\param_\delta}$ can be obtained:
\begin{align}
       \dot{h} &\geq \frac{\partial h_r}{\partial \bm{x}}\bigg(f(\bm{x}) + F(\bm{x})\lambda(\bm{x},\paramEst,t)+g(\bm{x})\bm{u}\bigg) - \eta + \frac{\partial h_r}{\partial t}\nonumber\\
       &\geq -\alpha\bigg(h_r(\bm{x},\paramEst,t) - \frac{1}{2}\paramErrorMax^{\top}\Gamma^{-1}\paramErrorMax + \gamma(\delta)\bigg) \nonumber \\
       &= -\alpha\bigg(h_r(\bm{x},\paramEst,t) - \frac{1}{2}\paramErrorMax^{\top}\Gamma^{-1}\paramErrorMax \bigg) -\frac{\epsilon\delta^2}{4} \nonumber \\
       &\geq -\alpha\bigg(h_r(\bm{x},\paramEst,t) - \frac{1}{2}\paramError^{\top}\Gamma^{-1}\paramError \bigg) -\frac{\epsilon\delta^2}{4} \nonumber 
\end{align}
which implies $ \dot{h} \geq -\alpha(h(\bm{x},\paramEst,t)) -\frac{\epsilon\delta^2}{4}$; thus any locally Lipschitz continuous controller $\bm{u}$ satisfying \eqref{prof:safe_cond} renders the system for \eqref{pf:sys_real_input_d} ISSf with respect to $\mathcal{S}^r_{\param_\delta}$ according to \eqref{pre:issfcbf_condition_1}.
\end{proof}
\begin{remark}
Note that $\epsilon > 0$ is a tunable parameter to confine $h_{\delta}$ to the smallest possible region. 
\end{remark}

Lastly, for the implementation of the proposed method, we formulate Quadratic Programming (QP) to provide the safe control input while satisfying the condition \eqref{pro:safe_constr} that is robust against the input disturbance and parametric model uncertainty as:
 \begin{align}
    &\bm{u}_{\mathsf{safe}} = \quad \argmin_{\bm{u}\in\mathcal{U}} \frac{1}{2} {\vert\vert \bm{u} - \bm{u}_{\mathsf{nominal}} \vert\vert}^2 \quad  \textbf{(RaTVCBF-QP)}\nonumber \\
    &\text{s.t.  } L_fh_r + L_Fh_r\lambda(\bm{x},\paramEst,t)+L_gh_r\bm{u} \geq \nonumber \\ 
    &-\alpha\bigg(h_r(\bm{x},\paramEst,t) - \frac{1}{2}\paramErrorMax^{\top}\Gamma^{-1}\paramErrorMax+\gamma(\delta)\bigg) +\eta - \frac{\partial h_r}{\partial t}.\label{RaISSF-qp}
\end{align}

%% file: application.tex
\section{Application to Robotic Surface Treatment}\label{sec:application}
In this section, it is shown that the proposed method ensures that MRR \eqref{Preston_equation} remains within an acceptable range across the entire surface. According to \cite{Li}, controlling MRR within the desired range ultimately achieves the desired quality guarantees. We verify the proposed method with numerical simulation on MATLAB and conduct experimental validation in a real robotic setup.
\subsection{Robotic Surface Treatment}
To ensure quality guarantees of the surface treatment via MRR as we mentioned in Section~\ref{sec:problem_formulation}, we enforce contact force constraints in \eqref{RaISSF-qp}. First, we define the parametric contact force, $\bm{f}_c \in \mathbb{R}^n$ from the Kelvin-Voigt model \cite{kelvinModel} as: 
\begin{equation}
\bm{f}_c=-diag(\bm{k})\bm{p}-diag(\bm{b})\dot{\bm{p}},
\label{eqn:kevin_voigt}
\end{equation}
where $diag(\cdot)$ is a diagonal matrix with the elements of a vector, and $\bm{p} \in \mathbb{R}^n_{\geq 0}$ is the penetration of the surface, and $m_o \in \mathbb{R}_{> 0}$, $\bm{k}\in \mathbb{R}^{n}_{\geq 0}$, and $\bm{b} \in \mathbb{R}^{n}_{\geq 0}$ are mass, stiffness and damping parameters, respectively. From \eqref{pf:sys_real_input_d}, the system model including unknown contact dynamics parameters of the model \eqref{eqn:kevin_voigt} is formulated as:
\begin{equation}  
\dot{\bm{x}} = \underbrace{\begin{bmatrix}
\dot{\bm{p}} \\
\bm{0}
\end{bmatrix}}_{f(\bm{x})} -\frac{1}{m_o}\underbrace{\begin{bmatrix}\bm{0}&\bm{0}\\diag(\bm{p}) & diag(\dot{\bm{p}})\end{bmatrix}}_{F(\bm{x})}\underbrace{\begin{bmatrix}
\bm{k} \\
\bm{b}
\end{bmatrix}}_{\theta^*}+\underbrace{\begin{bmatrix}
\bm{0}\\
\frac{1}{m_o}\bm{I}_{m}
\end{bmatrix}}_{g(\bm{x})}\bm{u}, \label{contact_system_model}
\end{equation}
where $\bm{x} = \begin{bmatrix} \bm{p} & \dot{\bm{p}}\end{bmatrix}^{\top}$ is the system state, and $\bm{u}\in \mathbb{R}^{m}$ is the direct force input generated by the force controller from the robot. $\bm{I}_m \in \mathbb{R}^{m\times m}$ denotes the $m\times m$ dimensional identity matrix where in this case, $n=m$. Note that we assume that there exists a stable force tracking controller in the robot (e.g. admittance or hybrid force/position controller). However, since the force controller might be inaccurate and have input disturbances but bounded, we define the whole force input to the system as $\bm{u} = \bm{u}_c + \bm{d}(t)$, where $\bm{u}_c \in  \mathbb{R}^m$ and $\bm{d}(t) \in \mathbb{R}^m$ are the control input from the force controller and input disturbance, respectively. 

Secondly, to ensure the desired contact pressure over the entire surface, we define a candidate control barrier function to regulate the contact force within admissible forces as follows: 
\begin{equation}
    h(\bm{x},t) =(\bm{f}_{c}(\bm{x}) - \bm{f}_{\mathsf{lower}}(t))(-\bm{f}_{c}(\bm{x}) + \bm{f}_{\mathsf{upper}}(t)) \geq 0, \label{eqn:safety_cons_min} 
\end{equation}
where $\bm{f}_{\mathsf{lower}}, \bm{f}_{\mathsf{upper}}$ are admissible time-varying lower and upper contact force bounds that can be determined by the desired MRR as shown in the upper right of Fig.~\ref{fig:scenario}, which depends on the experiment scenario setup.

\begin{figure}[t]
    \centering
    \includegraphics[width=1\linewidth]{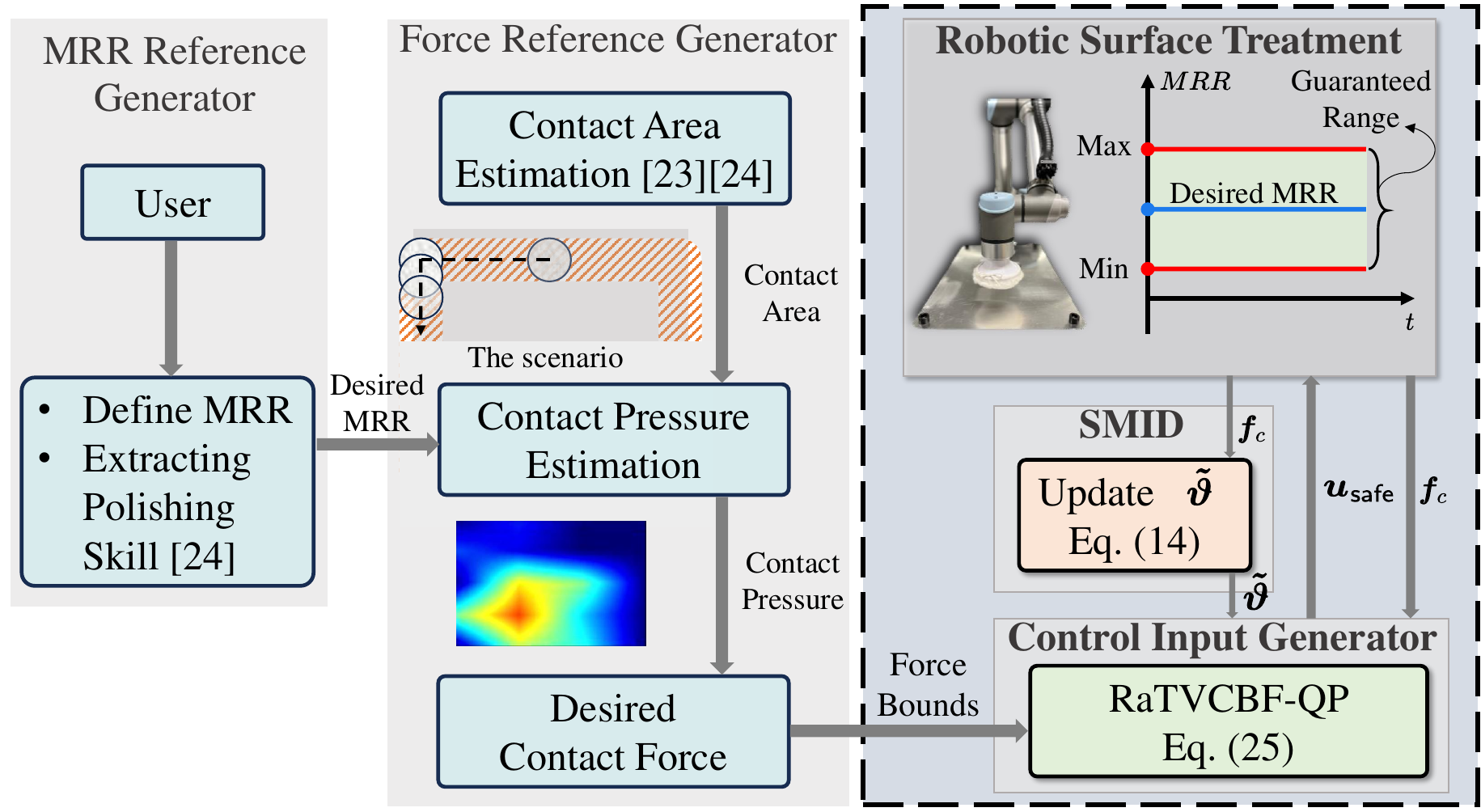}
    \caption{\small{Pipeline of the robotic surface treatment. First, a user defines the desired MRR. Subsequently, in the force reference generator, we estimate contact area between a tool and a given surface and then calculate the desired contact pressure based on the desired MRR. Afterwards, the desired force is obtained from the desired contact pressure by using \eqref{Preston_equation}. Finally, we can determine force bounds (e.g. $\pm15\%$ of the desired force) that are used as constraints to ensure the desired MRR in the dotted box.}}
    \vspace{-0.5cm}
    \label{fig:scenario}
\end{figure}
\begin{figure*}[t] 
\centering
\begin{subfigure}{0.32\textwidth}
    \centering
    \includegraphics[width=\textwidth]{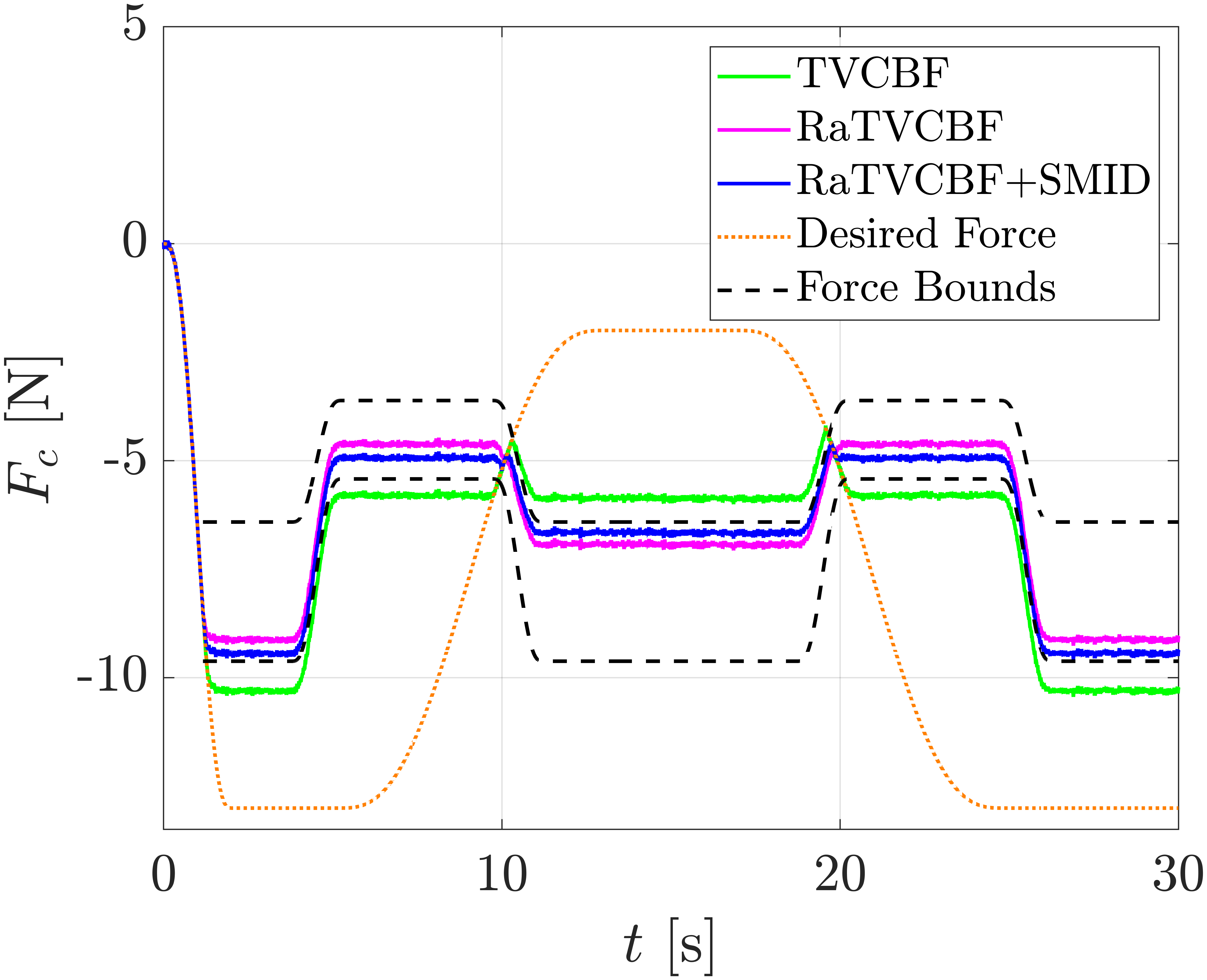}
    \caption{}
    \label{fig:force_bound}
\end{subfigure}
\hspace{0.5mm}
\begin{subfigure}{0.32\textwidth}
    \centering
    \includegraphics[width=\textwidth]{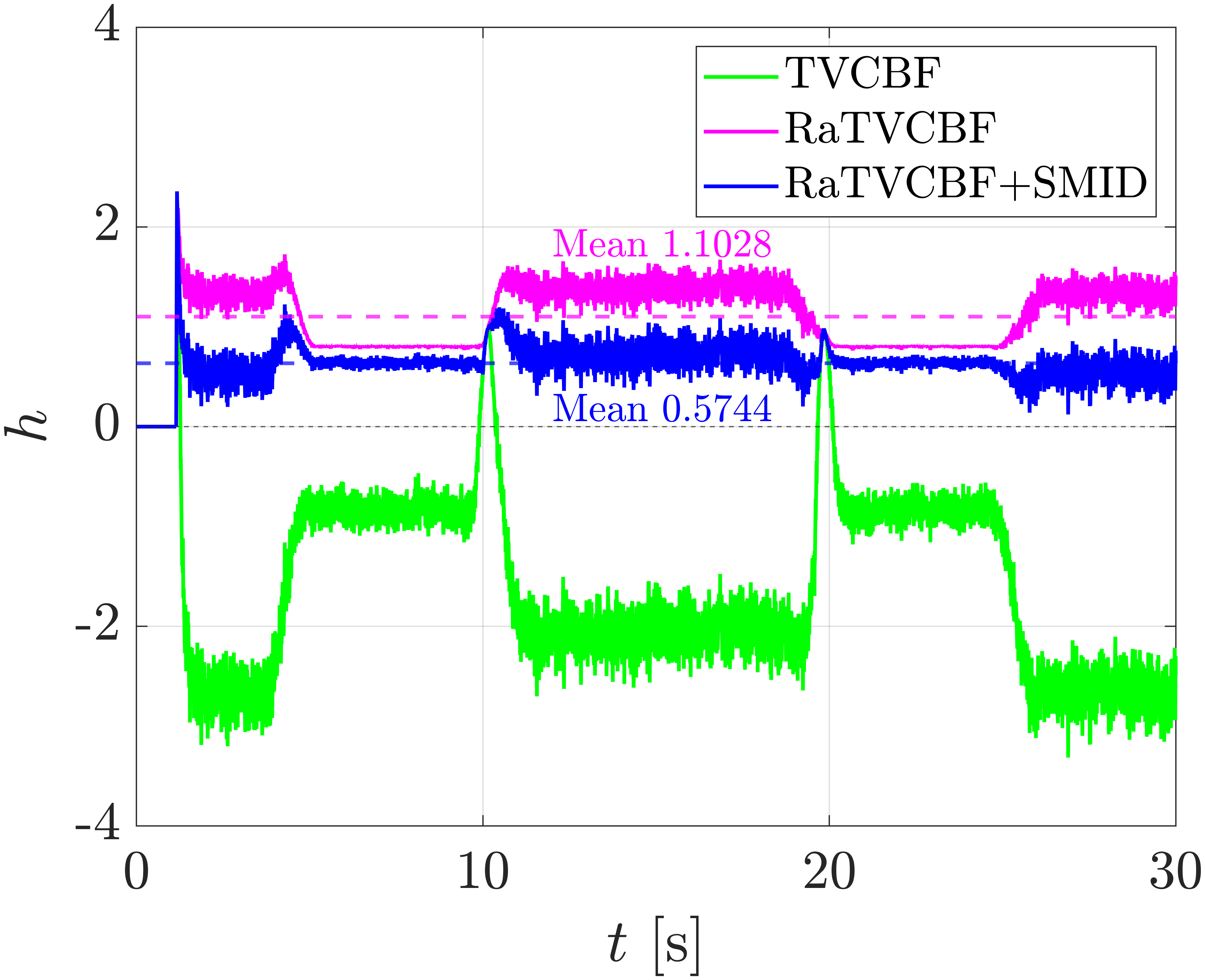}
    \caption{}
    \label{fig:h}
\end{subfigure}
\hspace{0.5mm}
\begin{subfigure}{0.32\textwidth}
    \centering
    \includegraphics[width=\textwidth]{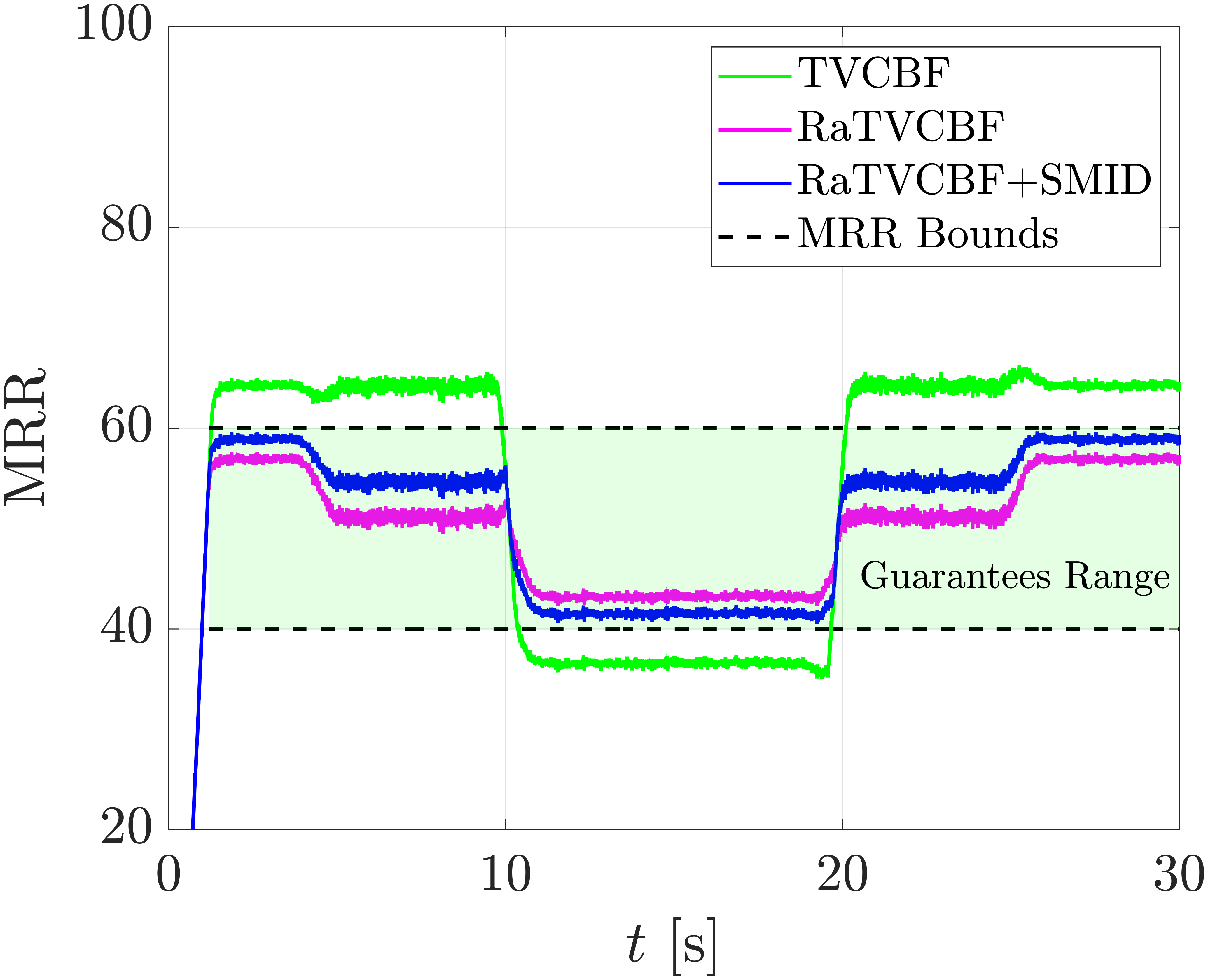}
    \caption{}
    \label{fig:mrr_bound}
\end{subfigure}

\caption{\small{The simulation results of quality guarantees with noise in the robotic polishing application. The figures show the performances of each CBFs-based controller in the presence of parametric model uncertainty and input disturbances in the system \eqref{contact_system_model}. (a) shows the current contact forces provided by each controller, including the admissible contact force bounds in dotted lines, and (b) represents $h$ values to check the constraint conditions; $h\geq 0$ satisfies the force bound constraints. (c) finally presents the performances of quality guarantees on each controller in terms of MRR. Note that we initiate the CBFs-based controllers after the contact forces reach the bound set as shown in (a).}}
\label{sim:graph}
\end{figure*}

\begin{figure*}[t] 
\centering
\begin{subfigure}{0.32\textwidth}
    \centering
    \includegraphics[width=\textwidth]{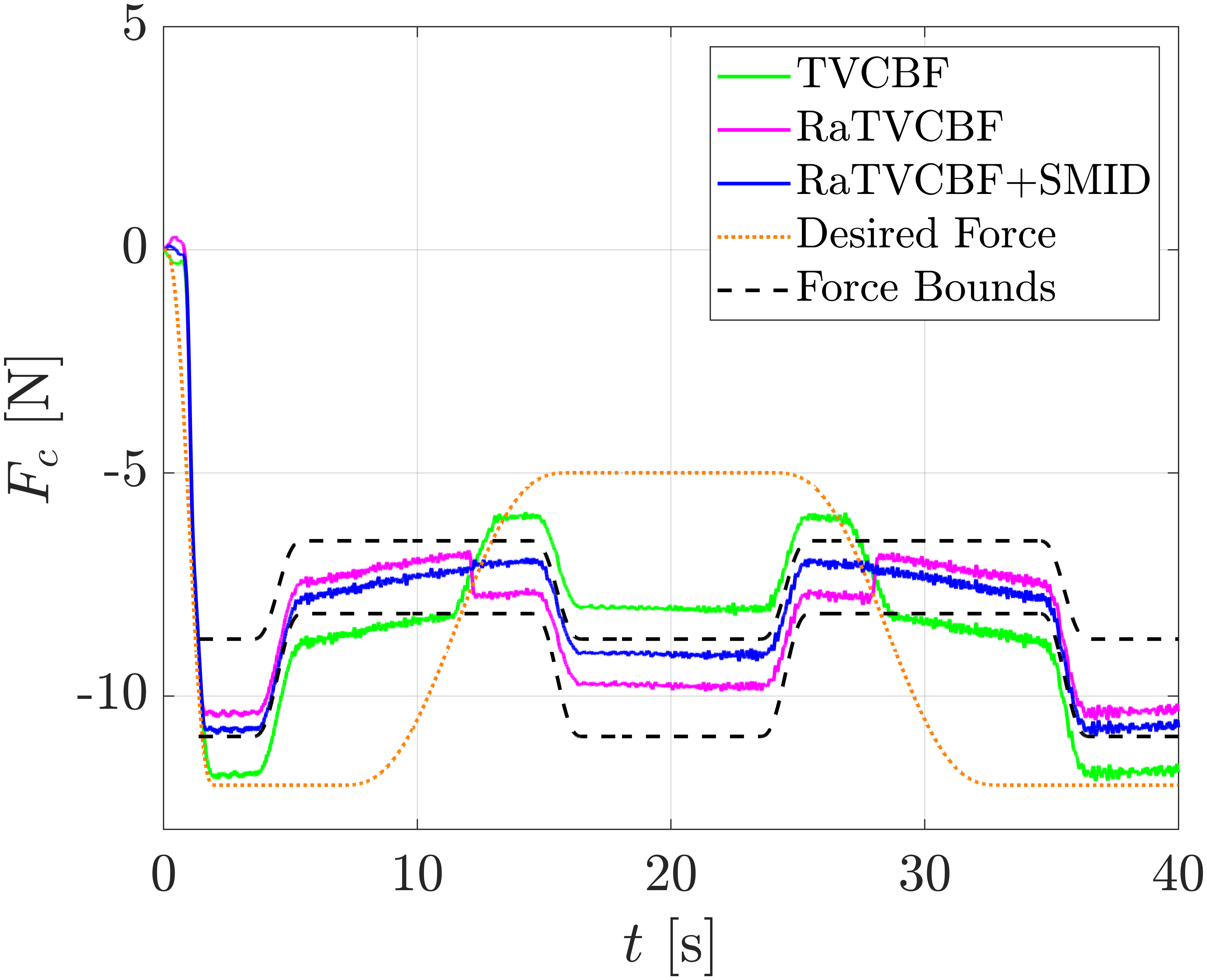}
    \caption{}
    \label{fig:force_bound_real}
\end{subfigure}
\hspace{0.5mm}
\begin{subfigure}{0.32\textwidth}
    \centering
    \includegraphics[width=\textwidth]{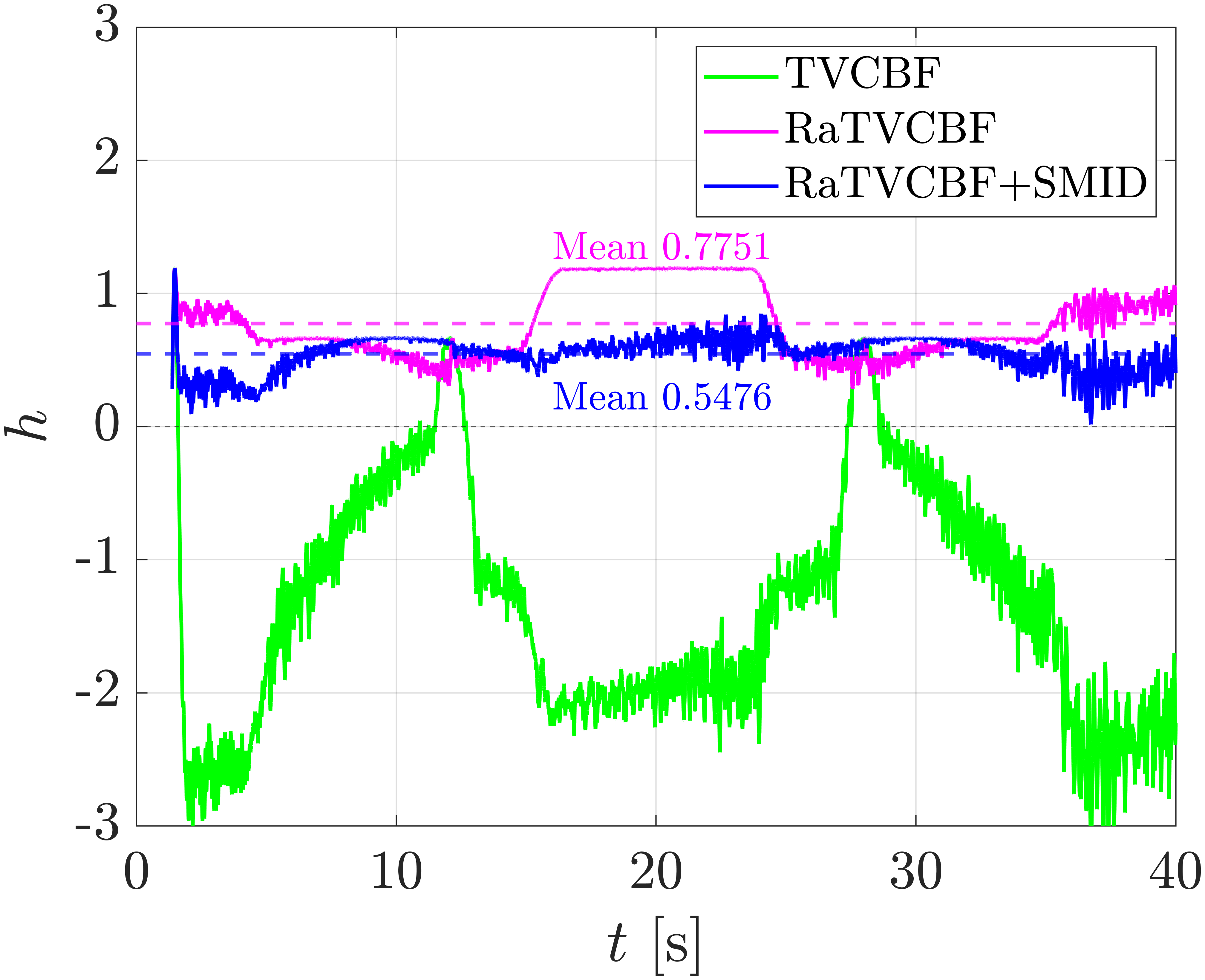}
    \caption{}
    \label{fig:h_real}
\end{subfigure}
\hspace{0.5mm}
\begin{subfigure}{0.32\textwidth}
    \centering
    \includegraphics[width=\textwidth]{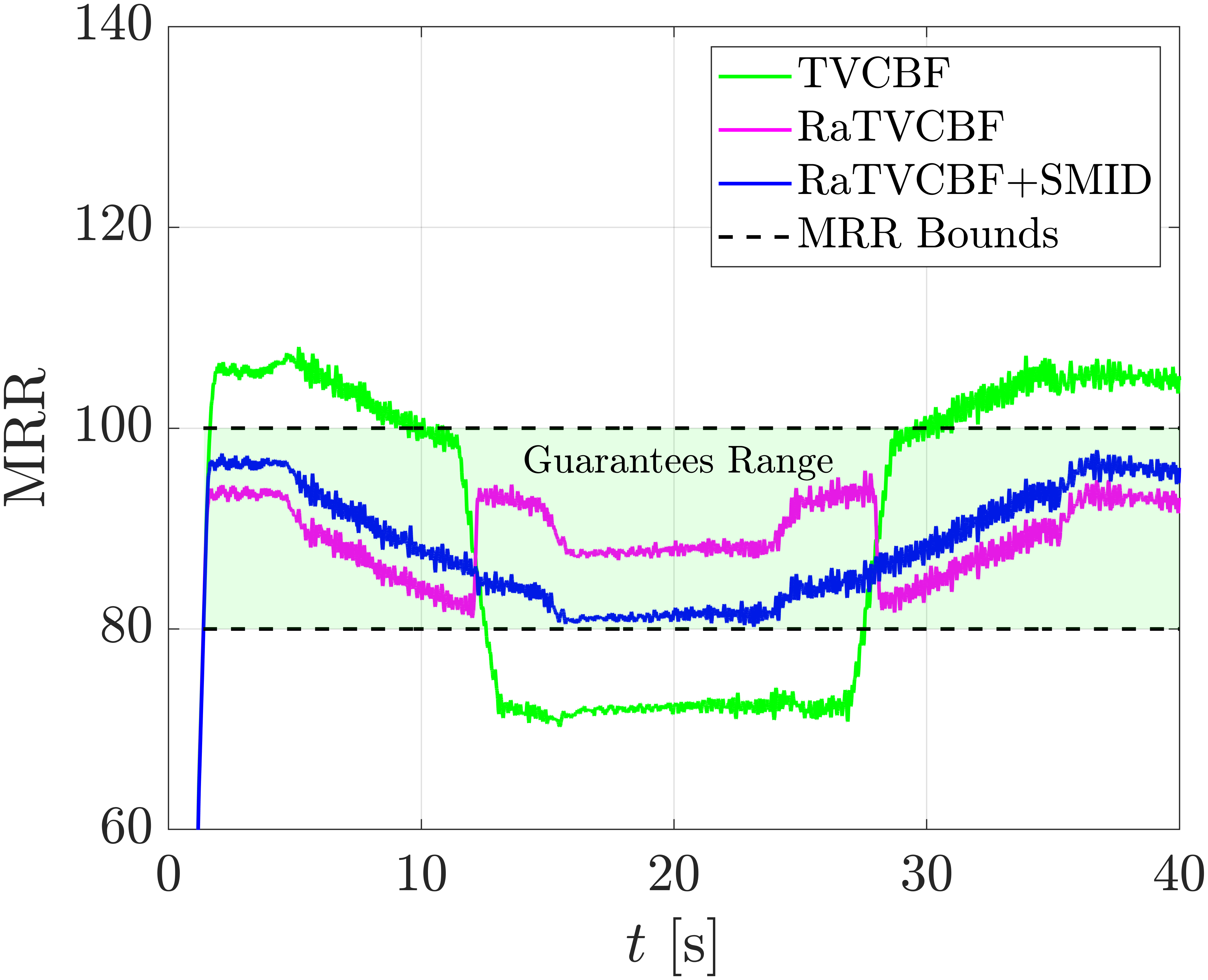}
    \caption{}
    \label{fig:mrr_bound_real}
\end{subfigure}
\caption{\small{Results of experimental validation in real robotic setup. (a) shows contact force trajectories of each method in the z-direction, and the desired force for the nominal controller is set to being outside of the force bounds on purpose to check if violations occur or not. (b) presents $h$ functions of each controller, and $h\geq 0$ indicates the current force belongs in the bound set. (c) shows the current MRR of each method with guarantees range in pale green. Note that each controller initiates when the contact happens and the current contact forces are in the desired constraint set.}}
\vspace{-0.3cm}
\label{real:graph_real}
\end{figure*}

\subsection{Scenario Setup for Simulation and Experiment}
In this work, we focus on the quality guarantees in the dotted box as shown in Fig.~\ref{fig:scenario}, rather than how to obtain MRR from the user and estimate the contact area. It is worth noting that obtaining MRR and contact area estimation can be done with analytical and numerical approaches \cite{XIAO2020188}\cite{YkIROS2022}. We arbitrarily select the desired MRR and design a simple but reliable plane-based scenario as shown in the middle of Fig.~\ref{fig:scenario}, so that the contact area can be easily calculated with the closed solution between the circle polishing tool and the plane. This allows us to show a clear comparison across CBFs-based controllers in the simulation and real robotic experiment. We first let a robot polish on the rectangular aluminium plane and pass edge areas and return to the start point after moving around the plane as shown in the middle of Fig.~\ref{fig:scenario}. On the edge regions, the contact areas are changed, which means that the current contact forces should be also adjusted to let the current MRR remain in \textit{the lower and upper MRR bounds} (e.g. at $\pm10$ of the desired MRR here). To this end, we calculate $\bm{f}_{\mathsf{lower}}$ and $ \bm{f}_{\mathsf{upper}}$ from the lower/upper MRR bounds by using \eqref{Preston_equation}, which is shown in Fig.~\ref{fig:force_bound}. In this scenario, we use a P controller as the nominal controller in the simulation. Note that we intentionally define the desired contact force profile for the nominal controller outside of the force bound, \eqref{eqn:safety_cons_min} to clearly show whether each controller ensures the bounds in the presence of uncertainties.

\subsection{Simulation Results}
Fig.~\ref{sim:graph} shows the simulation results of controllers in the robotic surface treatment scenario. We first carry out Time-Varying CBF (TVCBF) \cite{TvCBF2019} as a baseline in the simulation. In this simulation, we set the ground truth of $\bm{k},\bm{b}$ to $1400$ and $70$, respectively, on the other hand, the system model has the putative parameters, $\hat{\bm{k}},\hat{\bm{b}}$ that are $1000$ and $10$, respectively. Note that the proposed method does not rely on the initial estimate conditions such as over or under estimate of the parameters, as long as the initial conditions satisfy the maximum possible error condition. To show a clear comparison between controllers, we use the same $\mathcal{K}_{\infty}$ functions as $K(\cdot)$ and $\alpha(\cdot)$ in \eqref{pre:safety_condition} and \eqref{RaISSF-qp}, respectively. As shown in Fig.~\ref{fig:force_bound} and Fig.~\ref{fig:h}, TVCBF violates the constraint of force bound, \eqref{eqn:safety_cons_min} due to the uncertainties of model and input disturbances. On the other hand, the proposed method, RaTVCBF complies with the constraint, $h\geq 0$ since we impose the robustness against the parametric model uncertainty and input disturbances by using the maximum possible error, $\paramErrorMax$ and the bound of the disturbances, $\delta$, respectively. Since the method ensures that the contact forces stay within the force bounds, we can expect that it also enables the current MRR to remain within the desired lower and upper bounds of the desired MRR while the baseline violates them as shown in Fig.~\ref{fig:mrr_bound}. Fig.~\ref{fig:h} and \ref{fig:mrr_bound} show the performances of the proposed method with and without SMID. It is observed that applying SMID to RaTVCBF reduces the conservatism by approximately $48\%$ based on each mean value. We use the batch size of $5$ and $\varepsilon = 0.0008$ in this simulation.
\subsection{Experimental Validation}
To validate the proposed method, we carry out the robotic surface treatment on a flat aluminium plate by using a Universal Robot UR10e manipulator with a built-in F/T sensor and attach the polishing tool to the end-effector as shown in Fig.~\ref{fig:scenario}. The real experimental scenario is the same as the numerical simulation to show clearly that the experimental results are aligned with the simulation. In this experiment, we use a hybrid force/position controller as the nominal controller to track the desired path in x-y plane and contact force in the z-direction. 

We first generate the robot trajectory with constant linear velocity of the tool and calculate the contact area based on the area of circular segment in edge regions. Fig.~\ref{real:graph_real} shows the performances of each controller. The proposed method can regulate contact forces to stay within the force bounds compared to TVCBF as shown in Fig.~\ref{fig:force_bound_real}. The values of $h$ functions as shown in Fig.~\ref{fig:h_real} indicate clearly whether each method violates the force bounds. Since the proposed method allows the contact forces to maintain within \eqref{eqn:safety_cons_min}, the proposed method ensures the guaranteed range of MRR while the baseline violates the range as shown in Fig.~\ref{fig:mrr_bound_real}. This is because we enforce robustness against uncertainties by using the maximum possible error and the bound of input disturbances. Moreover, we use SMID to reduce the conservative behaviors of RaTVCBF based on the collected input-output data as shown in Fig.~\ref{fig:h_real} and \ref{fig:mrr_bound_real}. It is observed that leveraging SMID alleviates the conservatism by about $30\%$ when comparing each mean value. We have the batch size of $5$ and $\varepsilon = 0.35$ in this experiment. For practical implementation, $\varepsilon$ is a tunable parameter determined by the magnitude of disturbances, measurement noise, and unmodeled dynamics \cite{smidbook}. In addition, $k_p$ from \eqref{Preston_equation} is an unknown parameter, but it can be encapsulated in the parametric structure of $\theta$, (e.g. $\theta$ =$\begin{bmatrix}
        k_pk & k_pb
    \end{bmatrix}^{\top}$), so this uncertainty can also be treated in the proposed method. In the presence of tool wear, the Preston coefficient $k_p$ will be time-varying and may be estimated  based on the ratio between initial and final MRR \cite{Quinsat}. 


%% file: conclusions.tex
\section{Conclusion}\label{sec:conclusions}
This paper proposed a robust adaptive control design in the presence of parametric uncertainty and input disturbances, accompanying time-varying constraints. The design specifically leveraged RaCBFs and ISSf to make a controller robust against parametric uncertainty and disturbances, including time-varying control barrier functions. Furthermore, we incorporated SMID to decrease the conservatism of robust approaches. To evaluate the performance of the proposed method, we demonstrated it in the robotic surface treatment application where ensuring treatment quality is crucial. To ensure quality guarantees, we first defined MRR as the standard of quality based on Preston's equation and  then used the proposed method to keep the current MRR within the desired MRR range. The results of simulation and real experiment with clear comparison to an existing work showed that the proposed method outperforms the baseline in that it not only ensures quality guarantees but also achieves reduced conservatism despite uncertainties in the real robotic system. 
\section*{Acknowledgement}
This work was supported by MADE FAST, funded by
Innovation Fund Denmark, grant number 9090-00002B.